\documentclass[a4paper,11pt]{article}
\pdfoutput=1
\usepackage{a4wide}
\usepackage{amsmath,amsfonts,amssymb,amsthm}
\usepackage[colorlinks,citecolor=blue]{hyperref}
\usepackage{enumitem}
\usepackage{bbm}
\usepackage{natbib}
\usepackage{microtype}

\setcitestyle{square}

\newtheorem{theorem}{Theorem}[section]
\newtheorem{lemma}[theorem]{Lemma}
\newtheorem{proposition}[theorem]{Proposition}
\newtheorem{corollary}[theorem]{Corollary}

\theoremstyle{definition}

\newtheorem{assumption}[theorem]{Assumption}
\newtheorem{example}[theorem]{Example}

\DeclareMathOperator*{\argmin}{argmin}

\numberwithin{equation}{section}

\def \bE {\mathbb{E}}

\def \bI {\mathbb{I}}
\def \bJ {\mathbb{J}}

\def \bN {\mathbb{N}}

\def \bP {\mathbb{P}}

\def \bR {\mathbb{R}}
\def \bS {\mathbb{S}}

\def \bY {\mathbb{Y}}

\def \cB {\mathcal{B}}

\def \cH {\mathcal{H}}
\def \cI {\mathcal{I}}

\def \cW {\mathcal{W}}
\def \cX {\mathcal{X}}

\def \ran{\,{\rm ran}\,}
\def \ker{\,{\rm ker}\,}
\def \sp{\,{\rm span}\,}
\def \tr{\,{\rm tr}\,}

\begin{document}
\title{Sobolev norm inconsistency of kernel interpolation}
\author{
Yunfei Yang \thanks{School of Mathematics (Zhuhai) and Guangdong Province Key Laboratory of Computational Science, Sun Yat-sen University, Zhuhai, P.R. China. E-mail: \href{mailto:yangyunfei@mail.sysu.edu.cn}{yangyunfei@mail.sysu.edu.cn}.}
}
\date{}
\maketitle

\begin{abstract}
We study the consistency of minimum-norm interpolation in reproducing kernel Hilbert spaces corresponding to bounded kernels. Our main results give lower bounds for the generalization error of the kernel interpolation measured in a continuous scale of norms that interpolate between $L^2$ and the hypothesis space. These lower bounds imply that kernel interpolation is always inconsistent, when the smoothness index of the norm is larger than a constant that depends only on the embedding index of the hypothesis space and the decay rate of the eigenvalues.


\end{abstract}

\section{Introduction}

Recent empirical results in deep learning revealed that certain training algorithms can interpolate noisy data and yet generalize well \citep{belkin2019reconciling,zhang2017understanding}. This phenomenon, which is known as benign overfitting, posed a challenge to classical learning theory based on uniform convergence \citep{nagarajan2019uniform} and motivated many studies on understanding theoretical mechanisms for the good generalization performance of interpolating models \citep{bartlett2021deep,belkin2021fit}. In particular, conditions for benign overfitting have been theoretically established for many learning methods, such as linear regression \citep{bartlett2020benign,hastie2022surprises}, random features models \citep{mei2022generalization} and kernel minimum-norm interpolation \citep{liang2020just}. However, these results often rely on the assumption that the data dimension $d$ grows with the sample size $n$. In the classical low-dimensional regime, where $d$ is fixed, there is a line of works showing that kernel interpolation does not generalize well \citep{beaglehole2023inconsistency,buchholz2022kernel,haas2023mind,li2024kernel,rakhlin2019consistency}, which is contrary to the high-dimensional case \citep{liang2020just,lin2024kernel}. Thus, these works provide evidence that high dimensionality is necessary for benign overfitting.

In this paper, we complement these results by showing that kernel minimum-norm interpolation is always inconsistent in certain norms slightly stronger than the $L^2$-norm. Specifically, suppose we are given a set of $n$ samples $\{(x_i,y_i)\}_{i=1}^n$ independently generated from a probability distribution $(X,Y) \sim P$ on $\cX\times \bR$. The goal of non-parametric regression is to estimate the conditional mean $f^*_P(x) = \bE[Y|X=x]$ from the observed data. Let $\cH$ be a separable reproducing kernel Hilbert space (RKHS) on $\cX$ with a measurable bounded kernel $K:\cX \times \cX \to \bR$. Under certain smoothness assumption on the regression function $f^*_P$, it is well known that the kernel ridge regression
\begin{equation}\label{ridge}
\widehat{f}_{n,\lambda} = \argmin_{f\in \cH} \left\{ \frac{1}{n} \sum_{i=1}^n (f(x_i)-y_i)^2 + \lambda \|f\|_\cH^2 \right\}
\end{equation}
is minimax optimal, if the regularization parameter $\lambda>0$ is chosen properly according to the sample size $n$ \citep{fischer2020sobolev,lin2020optimal,steinwart2008support,steinwart2009optimal,zhang2024optimality}. In this paper, we are more interested in the limiting case $\lambda \to 0$, where one obtains the kernel minimum-norm interpolation
\begin{equation}\label{interpolation}
\widehat{f}_n = \widehat{f}_{n,0} = \argmin_{f\in \cH} \|f\|_\cH \quad \mbox{subject to} \quad f(x_i)=y_i,\ i=1,\dots,n.
\end{equation}
In contrast to the minimax optimality of kernel ridge regression, \citet{rakhlin2019consistency} showed that kernel interpolation is inconsistent for the Laplace kernel in odd dimensions, which was extended to all fixed dimension $d$ and kernels whose associated RKHS is a Sobolev space with smoothness $s\in (d/2, 3d/4)$ by \citet{buchholz2022kernel}. These results have been further generalized by \cite{haas2023mind} to dot-product kernels on the sphere $\bS^d$ whose RKHS is equivalent to a Sobolev space with smoothness $s>d/2$. When the training data are uniformly spaced, \citet{beaglehole2023inconsistency} established the inconsistency for translation-invariant kernels under certain spectral assumptions. Recently, \citet{li2024kernel} considered more general kernels whose eigenvalues decay as $\mu_i \asymp i^{-\beta}$ for some $\beta>1$. Here and below, for two quantities $A$ and $B$, $A \asymp B$ means $A \lesssim B \lesssim A$, while 
$A \lesssim B$ (or $B \gtrsim A$) means the statement that $A\le cB$ for some constant $c>0$. They proved that, for any $\epsilon>0$, the generalization error of kernel interpolation is lower bounded by $\Omega(n^{-\epsilon})$, if the kernel is H\"older continuous and the embedding index $\alpha^*$ (see (\ref{emb index}) below) satisfies $\alpha^*=1/\beta$. 

Note that the above related works measured the generalization error in $L^2(\nu)$ norm, where $\nu$ denotes the marginal distribution of $X$. We are going to study the generalization error in the $[\cH]_\nu^\gamma$-norm, where $[\cH]_\nu^\gamma$ is the interpolation space between RKHS $\cH$ and $L^2(\nu)$ with smoothness index $\gamma\in [0,1)$. In the special case of
Sobolev RKHS $\cH$, these norms coincide with fractional Sobolev norms between the used Sobolev space $\cH$ and $L^2$ (see Example \ref{example besov}). Besides, we consider kernels with eigenvalues decay $\mu_i \asymp (i(\log i)^\zeta)^{-\beta}$ for $\beta>1$ and $\zeta \in \bR$. Our main result gives lower bounds for the error $\|\widehat{f}_n - f_P^*\|_{[\cH]_\nu^\gamma}$, which show that kernel interpolation is always inconsistent in $[\cH]_\nu^\gamma$-norm when $\gamma>3(\alpha^*-1/\beta)$ and generalizes poorly when $\gamma=3(\alpha^*-1/\beta)$. In this result, one can view $3(\alpha^*-1/\beta)\ge 0$ as an index that quantifies the degree of inconsistency of Kernel interpolation. Note that, when $\alpha^*=1/\beta$, we almost get the strongest conclusion that Kernel interpolation is inconsistent in the $L^2$-norm. An interesting research direction is to explore whether this index can be reduced. As a special case of our result, we can recover the result of \citet{li2024kernel} when $\zeta =0$ and $\alpha^*=1/\beta$, and remove the H\"older continuity assumption on the kernel. Since the kernel ridge regression (\ref{ridge}) has been proven to be minimax optimal in $[\cH]_\nu^\gamma$-norm \citep{fischer2020sobolev,zhang2024optimality}, our results provide evidences for the necessity of tuning the regularization parameter in order to get minimax optimal rates.

The rest of the paper is organized as follows. In Section \ref{sec: preliminaries}, we give an introduction of function spaces related to RKHS. We present our main results in Section \ref{sec: main results}. Proofs of some technical lemmas are given in Section \ref{sec: proofs}.

\section{Preliminaries}\label{sec: preliminaries}

Let $\nu$ be a probability measure on a measurable space $\cX$. Let $\cH$ be a separable RKHS on $\cX$ with a measurable bounded kernel $K:\cX \times \cX \to \bR$. By \citep[Lemma 2.2 and 2.3]{steinwart2012mercer}, the natural embedding $I_\nu :\cH \to L^2(\nu)$, which maps $f\in\cH$ to its $\nu$-equivalence class $[f]_\nu$, is a Hilbert-Schmidt operator with Hilbert-Schmidt norm $\|I_\nu\|_{HS}^2 = \int_\cX K(x,x) d\nu(x)<\infty$. The adjoint operator $I_\nu^*: L^2(\nu) \to \cH$ is an integral operator with respect to the kernel:
\[
(I_\nu^*f)(x) = \int_\cX K(x,x') f(x')d\nu(x'),\quad f\in L^2(\nu).
\]
Next, we define the operators 
\[
T_\nu := I_\nu I_\nu^*: L^2(\nu) \to L^2(\nu),\quad \mbox{and} \quad C_\nu := I_\nu^* I_\nu: \cH \to \cH,
\]
which are self-adjoint, positive semi-definite and trace-class (thus compact) with trace norm $\|T_\nu\|_{Tr} = \|C_\nu\|_{Tr} = \|I_\nu\|_{HS}^2 = \|I_\nu^*\|_{HS}^2$. The spectral theorem for self-adjoint compact operators \citep[Lemma 2.12]{steinwart2012mercer} shows that 
\begin{equation}\label{spectral decomposition}
T_\nu = \sum_{i\in \bI} \mu_i \langle \cdot,[e_i]_\nu \rangle_{L^2(\nu)} [e_i]_\nu \quad \mbox{and} \quad C_\nu = \sum_{i\in \bI} \mu_i \langle \cdot,\sqrt{\mu_i} e_i \rangle_\cH \sqrt{\mu_i} e_i,
\end{equation}
where $\bI \subseteq \bN$ is a countable index set, $(\mu_i)_{i\in\bI} \subseteq (0,\infty)$ is a non-increasing summable sequence consisting of the non-zero eigenvalues (with geometric multiplicities) of $T_\nu$, and $(e_i)_{i\in\bI} \subseteq \cH$ is a family of functions such that $([e_i]_\nu)_{i\in\bI}$ form an orthonormal basis of $\overline{\ran I_\nu} \subseteq L^2(\nu)$ and $(\sqrt{\mu_i}e_i)_{i\in\bI}$ is an orthonormal basis of $(\ker I_\nu)^\perp \subseteq \cH$. For convenience, we will assume that $\bI=\bN$ in the rest of the paper.

The spectral decomposition of $T_\nu$ allows us to define power spaces  introduced by \citep{steinwart2012mercer}. For $s\ge 0$, we define the power operator $T_\nu^s: L^2(\nu) \to L^2(\nu)$ by
\[
T_\nu^s = \sum_{i=1}^\infty \mu_i^s \langle \cdot,[e_i]_\nu \rangle_{L^2(\nu)} [e_i]_\nu.
\]
Then, the $s$-power space is defined by
\[
[\cH]^s_\nu := \ran T_\nu^{s/2} = \left\{ \sum_{i=1}^\infty a_i \mu_i^{s/2} [e_i]_\nu: \sum_{i=1}^\infty a_i^2<\infty \right\} \subseteq L^2(\nu),
\]
equipped with the norm
\[
\left\| \sum_{i=1}^\infty a_i \mu_i^{s/2} [e_i]_\nu \right\|_{[\cH]_\nu^s} = \left(\sum_{i=1}^\infty a_i^2\right)^{1/2}.
\]
In the rest of this paper, we will use the abbreviation $\|\cdot\|_s:= \|\cdot\|_{[\cH]_\nu^s}$ for simplicity. It is easy to see that $[\cH]_\nu^s$ is a separable Hilbert space with orthonormal basis $(\mu_i^{s/2}[e_i]_\nu)_{i\in\bN}$. Moreover, we have compact embedding $[\cH]_\nu^s \hookrightarrow [\cH]_\nu^r$ for $s > r \ge 0$ with $[\cH]_\nu^0 = \overline{\ran I_\nu} \subseteq L^2(\nu)$ equipped with the norm $\|\cdot\|_0=\|\cdot\|_{L^2(\nu)}$ and $[\cH]_\nu^1 = \ran I_\nu$, which is isometrically isomorphic to $(\ker I_\nu)^\perp \subseteq \cH$, i.e. $\|[f]_\nu\|_1 = \|f\|_\cH$ for $f\in (\ker I_\nu)^\perp$. We remark that $[\cH]_\nu^s$ equals the real interpolation space $[L^2(\nu),[\cH]_\nu^1]_{s,2}$ with equivalent norms \citep[Theorem 4.6]{steinwart2012mercer}. The following embedding property plays a key role in the error analysis of kernel methods.

\begin{assumption}[Embedding property]\label{EMB}
There exists $0<\alpha\le 1$ such that $[\cH]_\nu^\alpha$ is continuously embedded into $L^\infty(\nu)$, in other words,
\[
M_\alpha:=\| [\cH]_\nu^\alpha \hookrightarrow L^\infty(\nu)\|<\infty.
\]
\end{assumption}

It was shown by \citet[Theorem 9]{fischer2020sobolev} that the operator norm of the embedding can be computed by
\begin{equation}\label{emb norm}
M_\alpha^2= \left\| \sum_{i=1}^\infty \mu_i^\alpha e_i^2 \right\|_{L^\infty(\nu)}.
\end{equation}
Recall that we assume the kernel $K$ is bounded, which implies $[\cH]_\nu^1$ is continuously embedded into $L^\infty(\nu)$. Thus, Assumption \ref{EMB} always holds for $\alpha=1$. The embedding index $\alpha^*$ is the smallest $\alpha$ that satisfies Assumption \ref{EMB}. More precisely,
\begin{equation}\label{emb index}
\alpha^* := \inf \{0<\alpha\le 1: M_\alpha<\infty \}.
\end{equation}
We remark that the embedding property is related to the effective dimension 
\begin{equation}\label{effective dim}
N_\nu(\lambda) := \tr((C_\nu+\lambda)^{-1}C_\nu) = \sum_{i=1}^\infty \frac{\mu_i}{\mu_i+\lambda},\quad \lambda>0,
\end{equation}
where $\tr$ denotes the trace operator. The effective dimension is widely used as a complexity measurement of RKHS in the analysis of kernel-based methods \citep{blanchard2018optimal,caponnetto2007optimal,lin2020optimal,lin2017distributed}. As shown by \citep[Theorem 5.3]{steinwart2012mercer}, $M_\alpha<\infty$ implies that $\sum_{i=1}^\infty \mu_i^\alpha<\infty$. Applying Proposition \ref{appendix bound 2} in the Appendix, we get
\begin{equation}\label{effective dim bound}
N_\nu(\lambda) \lesssim \lambda^{-\alpha}.
\end{equation}

\section{Main results}\label{sec: main results}

Let $\cX$ be a measurable space, $P$ be an unknown probability distribution on $\cX \times \bR$ with
\[
\int_{\cX \times \bR} y^2 dP(x,y)<\infty.
\]
We denote the marginal distribution of $P$ on $\cX$ by $\nu$, and the conditional probability by $P(\cdot|x)$. The conditional mean is $\nu$-almost everywhere well defined and denoted by
\[
f^*_P(x) = \int_\bR y dP(y|x).
\]
We make the following assumption on the noise $y-f_P^*(x)$.

\begin{assumption}[Noise condition]\label{ass noise}
There exists a $\sigma>0$ such that for $\nu$-almost all $x\in \cX$,
\[
\int_\bR (y-f_P^*(x))^2 dP(y|x) \ge \sigma^2.
\]
\end{assumption}

Suppose that we have a set of $n$ samples $\{(x_i,y_i)\}_{i=1}^n$ independently generated from $P$. For convenience, let us denote the input and output samples by $X = (x_1,\dots, x_n)^\intercal \in \cX^n$ and $Y=(y_1,\dots,y_n)^\intercal \in \bR^n$. Let $\cH$ be a separable RKHS on $\cX$ with a measurable bounded kernel $K:\cX \times \cX \to \bR$. We introduce the notations $K_x(\cdot):=K(\cdot,x) \in \cH$, $K(x,X) := (K(x,x_1),\dots,K(x,x_n))^\intercal \in \bR^n$ and $K(X,X) :=(K(x_i,x_j))_{i,j=1}^n \in \bR^{n\times n}$. The subspace of $\cH$ generated by $K_{x_i}$ is denoted by $\cH_X := \sp \{K_{x_1},\dots,K_{x_n} \} \subseteq \cH$. 

It is well known that the kernel ridge regression $\widehat{f}_{n,\lambda}$ defined by (\ref{ridge}) has an explicit formula \citep[Section 12.5]{wainwright2019high}:
\begin{equation}\label{rep}
\widehat{f}_{n,\lambda}(x) = K(x,X)^\intercal (K(X,X)+n\lambda)^{-1} Y \in \cH_X,
\end{equation}
where we assume $K(X,X)$ is invertible when $\lambda=0$. There is an alternative representation of $\widehat{f}_{n,\lambda}$ by using integral operators \citep{caponnetto2007optimal,fischer2020sobolev}. More precisely, let $\nu_n = \frac{1}{n}\sum_{i=1}^n \delta_{x_i}$ be the empirical distribution. By the definition of the operator $C_{\nu_n} = I_{\nu_n}^* I_{\nu_n}:\cH \to \cH$, we have $C_{\nu_n} K_{x_j} = \frac{1}{n} \sum_{i=1}^n K(x_i,x_j)K_{x_i}$. If we denote $(a_1,\dots,a_n)^\intercal =  (K(X,X)+n\lambda)^{-1} Y$, then 
\begin{align*}
(C_{\nu_n} + \lambda) \widehat{f}_{n,\lambda} = \frac{1}{n}  \sum_{i=1}^n \left(\sum_{j=1}^na_j K(x_i,x_j)+ n\lambda  a_i\right)K_{x_i} = \frac{1}{n} \sum_{i=1}^n y_i K_{x_i}.
\end{align*}
For $\lambda>0$, the operator $C_{\nu_n} + \lambda$ is invertible on $\cH$, which implies
\begin{equation}\label{operator rep}
\widehat{f}_{n,\lambda} = \frac{1}{n} \sum_{i=1}^n y_i (C_{\nu_n} + \lambda)^{-1} K_{x_i}.
\end{equation}
For the kernel minimum-norm interpolation $\widehat{f}_n=\widehat{f}_{n,0}$ defined by (\ref{interpolation}), the operator $C_{\nu_n}$ is invertible on $\cH_X$, because its matrix representation under the natural basis $\{K_{x_1},\dots,K_{x_n}\}$ is $n^{-1}K(X,X)$. Hence, the representation (\ref{operator rep}) is still true for $\lambda=0$, if $C_{\nu_n}^{-1}$ is understood as the inverse of the restriction $C_{\nu_n}|_{\cH_X}$ (and we will use this convention in the following).


Our goal is to estimate the generalization error of the minimum-norm interpolation $\widehat{f}_n$ in the $[\cH]_\nu^\gamma$-norm, that is $\|[\widehat{f}_n]_\nu - f_P^*\|_\gamma$. The starting point of our analysis is the following lemma, which lower bounds the error of any estimator by its variance. The proof is given in Subsection \ref{sec: lower bound by variance}.

\begin{lemma}\label{lower bound by variance}
Let $f\in \cH$ be any estimator based on the dataset $(X,Y)$ of the form
\[
f(x) = \frac{1}{n}\sum_{i=1}^n y_i u_i(x),
\]
where $u_i \in \cH$, $i=1,\dots,n$, only depend on $X$. If Assumption \ref{ass noise} holds and $f_P^*\in [\cH]_\nu^\gamma$ for some $0\le \gamma \le 1$, then almost surely for all $X\in \cX^n$,
\[
\bE \left[\|[f]_\nu - f_P^*\|_\gamma^2 |X \right] \ge \frac{\sigma^2}{n^2} \sum_{i=1}^n \|[u_i]_\nu\|_\gamma^2. 
\]
\end{lemma}

By the representation (\ref{operator rep}), we can apply Lemma \ref{lower bound by variance} to get
\begin{equation}\label{v lambda}
\bE \left[ \left\|[\widehat{f}_{n,\lambda}]_\nu - f_P^*\right\|_\gamma^2 |X \right] \ge \frac{\sigma^2}{n^2} \sum_{i=1}^n \left\|[(C_{\nu_n} + \lambda)^{-1} K_{x_i}]_\nu \right\|_\gamma^2 =:\sigma^2 V(\lambda). 
\end{equation}
Since $C_{\nu_n}$ is self-adjoint positive definite on $\cH_X$, we have for $\lambda_1\ge \lambda_2\ge 0$,
\begin{align*}
\left\|[(C_{\nu_n} + \lambda_2)^{-1} K_{x_i}]_\nu \right\|_\gamma^2 &= \left\langle (C_{\nu_n} + \lambda_2)^{-2} K_{x_i}, K_{x_i} \right\rangle_\gamma \\
&\ge \left\langle (C_{\nu_n} + \lambda_1)^{-2} K_{x_i}, K_{x_i} \right\rangle_\gamma \\
&= \left\|[(C_{\nu_n} + \lambda_1)^{-1} K_{x_i}]_\nu \right\|_\gamma^2,
\end{align*}
where $\langle \cdot,\cdot \rangle_\gamma$ denotes the inner product on $\cH_X$ induced by the norm $\|\cdot\|_\gamma$, and we use the partial ordering of the positive semi-definite operators $(C_{\nu_n} + \lambda_2)^{-2} \succeq (C_{\nu_n} + \lambda_1)^{-2}$. We conclude that $V(\lambda)$ is non-increasing. Hence, in order to lower bound the error of the minimum-norm interpolation $\widehat{f}_{n,0}$, we only need to lower bound $V(\lambda)$ for small $\lambda>0$, where concentration results can be applied.

When $0\le \gamma <1$, we can compute the $\gamma$-norm $\|[f]_\nu\|_\gamma$ of a function $f\in \cH$ by the $\cH$-norm $\| C_\nu^{(1-\gamma)/2} f\|_\cH$ (see Lemma \ref{norm eq} in the Appendix or \citep[Lemma 12]{fischer2020sobolev}). This implies that
\[
V(\lambda) = \frac{1}{n^2} \sum_{i=1}^n \left\|C_\nu^{(1-\gamma)/2} (C_{\nu_n} + \lambda)^{-1} K_{x_i} \right\|_\cH^2.
\]
We estimate $V(\lambda)$ by the following two-step approximation in high probability:
\begin{align*}
V(\lambda) &\approx V_1(\lambda):= \frac{1}{n^2} \sum_{i=1}^n \left\|C_\nu^{(1-\gamma)/2} (C_\nu + \lambda)^{-1} K_{x_i} \right\|_\cH^2 \\
&\approx V_2(\lambda):= \frac{1}{n} \int_\cX \|C_\nu^{(1-\gamma)/2} (C_\nu + \lambda)^{-1} K_x \|_\cH^2 d\nu(x).
\end{align*}
Note that our approximation of $V(\lambda)$ is slightly different from the approximation used in \citet[Supplementary Materials]{li2024kernel}. This helps us remove the regularity assumption on the kernel in \citet{li2024kernel}. We estimate the approximation errors in the following two lemmas, whose proofs are deferred to Subsections \ref{sec: v bound 1} and \ref{sec: v bound 2} respectively.

\begin{lemma}\label{v bound 1}
Let $0\le \gamma <1$, $n\ge 1$, $0<\kappa\le 1$ and $\tau \ge 1$. If Assumptions \ref{EMB} holds for some $\alpha\in (0,1]$, then there exists a constant $c>0$, independent of $n$, $\kappa$ and $\tau$, such that for any $\lambda\in (0,1/2)$ satisfying $\lambda^{-\alpha} \log \lambda^{-1} \le cn \kappa^2 \tau^{-1}$, the following inequality holds with probability at least $1-2e^{-\tau}$,
\[
|V(\lambda) - V_1(\lambda)| \le \frac{\kappa M_\alpha^2}{n \lambda^{\gamma+\alpha}}.
\]
\end{lemma}

\begin{lemma}\label{v bound 2}
Let $0\le \gamma <1$, $n\ge 1$, $0<\lambda< 1/2$ and $\tau \ge 1$. It holds that
\begin{equation}\label{V2}
V_2(\lambda) = \frac{1}{n}\sum_{i=1}^\infty \frac{\mu_i^{2-\gamma}}{(\mu_i+\lambda)^2}.
\end{equation}
If Assumption \ref{EMB} holds, then with probability at least $1-2e^{-\tau}$,
\[
|V_1(\lambda)- V_2(\lambda)| \le \frac{\sqrt{\tau }M_\alpha^2}{\sqrt{2} n^{3/2}\lambda^{\gamma+\alpha}}.
\]
\end{lemma}

Note that the condition $\lambda^{-\alpha} \log \lambda^{-1} \le cn \kappa^2 \tau^{-1}$ in Lemma {\ref{v bound 1}} prevents us from choosing a too small $\lambda$. This is also the reason why we need to lower bound $V(0)$ by $V(\lambda)$, rather than estimating $V(0)$ directly. Because we are not able to use concentration inequality to estimate the difference $|V(0) - V_1(0)|$. In order to lower bound $V_2(\lambda)$, we make the following assumption on the eigenvalues, which is slightly more general than assumptions on related works \citep{fischer2020sobolev,li2024kernel}

\begin{assumption}[Eigenvalue decay]\label{EVD}
There exist $\beta> 1$ and $\zeta\in \bR$ such that $\mu_i \asymp (i(\log i)^\zeta)^{-\beta}$ for $i\ge 2$.
\end{assumption}

It was shown by \citet[Theorem 5.3]{steinwart2012mercer} that if $[\cH]_\nu^\alpha$ is continuously embedded in $L^\infty(\nu)$, then the eigenvalues satisfy $\sum_{i=1}^\infty \mu_i^\alpha <\infty$ (which implies $\mu_i \lesssim i^{-1/\alpha}$). Hence, under Assumption \ref{EVD}, we always have $1/\beta \le \alpha^*\le 1$. The next lemma gives a lower bound for $V_2(\lambda)$ under Assumption \ref{EVD}.

\begin{lemma}\label{lower bound v2}
If Assumption \ref{EVD} holds, then, for $0\le \gamma<1$ and $0<\lambda <1/2$,
\[
\sum_{i=1}^\infty \frac{\mu_i^{2-\gamma}}{(\mu_i+\lambda)^2} \gtrsim \lambda^{-\gamma-1/\beta}(\log \lambda^{-1})^{-\zeta}.
\]
\end{lemma}

Combining the above analysis, we are able to give lower bounds for the generalization error of the kernel interpolation $\widehat{f}_n$, when the error is measured in the $[\cH]_\nu^\gamma$-norm. We summarize our main result in the following theorem.

\begin{theorem}\label{main theorem}
Let $\cH$ be a separable RKHS on a measurable space $\cX$ with a bounded measurable kernel $K$, and $P$ be probability distribution on $\cX \times \bR$ with marginal distribution $\nu$ such that Assumptions \ref{ass noise} and \ref{EVD} hold. Let $0\le\gamma<1$, $\alpha^*$ be the embedding index defined by (\ref{emb index}) and $\widehat{f}_n$ be the kernel interpolation (\ref{interpolation}) based on the random samples $(X,Y) \in (\cX \times \bR)^n$ from $P^n$. If $f_P^*\in [\cH]_\nu^\gamma$, then, for any $\epsilon>0$ and $\delta\in (0,1)$, when $n$ is sufficiently large (depending on $\epsilon$ and $\delta$), the lower bound
\[
\bE \left[ \left\|[\widehat{f}_n]_\nu - f_P^*\right\|_\gamma^2 |X \right] \gtrsim \sigma^2 n^{\frac{\gamma -3(\alpha^*-1/\beta)}{3\alpha^* -2/\beta} -\epsilon}
\]
holds with probability at least $1-\delta$ with respect to the samples $X$. If $\alpha^*$ also satisfies Assumption \ref{EMB}, i.e. $M_{\alpha^*}<\infty$, then the lower bound can be improved to
\[
\bE \left[ \left\|[\widehat{f}_n]_\nu - f_P^*\right\|_\gamma^2 |X \right] \gtrsim \sigma^2 n^{\frac{\gamma -3(\alpha^*-1/\beta)}{3\alpha^* -2/\beta}} (\log n)^{- \frac{\gamma+1/\beta}{3\alpha^* -2/\beta}(1+2\zeta) -\epsilon}.
\]
\end{theorem}

\begin{proof}
Let $0<\kappa\le 1$, which will be chosen later. By inequality (\ref{v lambda}), we only need to lower bound $V(\lambda)$ for some small $\lambda$. For any $\alpha$ that satisfies Assumption \ref{EMB}, we can apply Lemma \ref{v bound 1} and Lemma \ref{v bound 2}, which show the existence of a constant $c_1>0$, independent of $n$, $\kappa$ and $\tau$, such that for any $\lambda\in (0,1/2)$ satisfying 
\begin{equation}\label{lambda ineq}
\lambda^{-\alpha} \log \lambda^{-1} \le c_1n \kappa^2 \tau^{-1},
\end{equation}
the following lower bounds hold with probability at least $1-4e^{-\tau}$,
\begin{align*}
V(\lambda) &\ge V_1(\lambda) - \frac{\kappa M_\alpha^2}{n \lambda^{\gamma+\alpha}} \\
&\ge V_2(\lambda) - \frac{\sqrt{\tau }M_\alpha^2}{\sqrt{2} n^{3/2}\lambda^{\gamma+\alpha}} - \frac{\kappa M_\alpha^2}{n \lambda^{\gamma+\alpha}}.
\end{align*} 
By the expression (\ref{V2}) and Lemma \ref{lower bound v2}, we further get
\begin{align*}
V(\lambda) &\ge \frac{c_2 M_\alpha^2}{n \lambda^{\gamma+1/\beta} (\log \lambda^{-1})^\zeta} - \frac{\sqrt{\tau }M_\alpha^2}{\sqrt{2} n^{3/2}\lambda^{\gamma+\alpha}} - \frac{\kappa M_\alpha^2}{n \lambda^{\gamma+\alpha}} \\
&= \left( c_2 \lambda^{\alpha-1/\beta} (\log \lambda^{-1})^{-\zeta} - \frac{\sqrt{\tau}}{\sqrt{2n}} - \kappa \right) \frac{M_\alpha^2}{n\lambda^{\gamma+\alpha}},
\end{align*} 
for some constant $c_2>0$. For fixed $\tau\ge 1$ and sufficiently large $n$, if we let $\kappa = \sqrt{\tau} n^{-t}(\log n)^{-s}$ with $t\in [0,1/2)$ and $s\in \bR$ such that $\kappa\le 1$, then we can choose $\lambda = c_3 (n^{2t-1} (\log n)^{1+2s})^{1/\alpha}$ for some $c_3>0$ so that it satisfies the inequality (\ref{lambda ineq}). Observe that, for $t\in [0,1/2)$,
\begin{align*}
&c_2 \lambda^{\alpha-1/\beta} (\log \lambda^{-1})^{-\zeta} - \frac{\sqrt{\tau}}{\sqrt{2n}} - \kappa \\
\gtrsim &\ n^{(2t-1)(1-\frac{1}{\alpha \beta})} (\log n)^{(1+2s)(1-\frac{1}{\alpha \beta})-\zeta} - \sqrt{\tau} n^{-t} (\log n)^{-s}.
\end{align*}
If we choose $t= \frac{\alpha \beta -1}{3\alpha\beta -2} \in [0,1/3)$ and $s> \frac{1-\alpha\beta}{3\alpha\beta -2} + \frac{\alpha\beta}{3\alpha\beta -2} \zeta$, then the first term dominates the second term in the above inequality for sufficiently large $n$. In this case,
\[
\lambda = c_3 n^{-\frac{\beta}{3\alpha \beta-2}} (\log n)^{(1+2s)/\alpha},
\]
and
\begin{align*}
V(\lambda) &\gtrsim n^{-1} \lambda^{-\gamma -1/\beta} (\log \lambda^{-1})^{-\zeta} \\
&\gtrsim n^{\frac{\gamma \beta +1}{3\alpha \beta -2}-1} (\log n)^{-(1+2s)\frac{\gamma\beta +1}{\alpha \beta}-\zeta}.
\end{align*}

Recall that we can apply the above analysis for any $\alpha$ that satisfies Assumption \ref{EMB} (thus any $\alpha>\alpha^*$). Since the exponent $\frac{\gamma \beta +1}{3\alpha \beta -2}$ is continuous and decreasing on $\alpha$, for any $\epsilon>0$, we can choose $\alpha>\alpha^*$ such that $\frac{\gamma \beta +1}{3\alpha \beta -2} \ge \frac{\gamma \beta +1}{3\alpha^* \beta -2} -\frac{\epsilon}{2}$. For such $\alpha$, we have
\begin{align*}
V(\lambda) &\gtrsim n^{\frac{\gamma \beta +1}{3\alpha \beta -2}-1} (\log n)^{-(1+2s)\frac{\gamma\beta +1}{\alpha \beta}-\zeta} \\
&\gtrsim n^{\frac{\gamma \beta +1}{3\alpha^* \beta -2} - 1-\epsilon} \asymp n^{\frac{\gamma -3(\alpha^*-1/\beta)}{3\alpha^* -2/\beta} -\epsilon}.
\end{align*}

If $\alpha^*$ also satisfies Assumption \ref{EMB}, then we can take $\alpha = \alpha^*$ in the above analysis and hence
\begin{align*}
V(\lambda) &\gtrsim n^{\frac{\gamma \beta +1}{3\alpha^* \beta -2}-1} (\log n)^{-(1+2s)\frac{\gamma\beta +1}{\alpha^* \beta}-\zeta} \\
&\asymp n^{\frac{\gamma -3(\alpha^*-1/\beta)}{3\alpha^* -2/\beta}} (\log n)^{-(1+2s)\frac{\gamma\beta +1}{\alpha^* \beta}-\zeta}.
\end{align*}
Noticing that we can choose any $s> \frac{1-\alpha^*\beta}{3\alpha^*\beta -2} + \frac{\alpha^*\beta}{3\alpha^*\beta -2} \zeta$ and hence any exponent 
\[
-(1+2s)\frac{\gamma\beta +1}{\alpha^* \beta} - \zeta <- \frac{\gamma\beta+1}{3\alpha^*\beta -2}(1+2\zeta),
\] 
we finishes the proof.
\end{proof}

Note that we implicitly assume that the dimension $d$ of the input space is fixed in the above analysis. In the high dimensional case, where $d$ grows with the sample size $n$, Theorem \ref{main theorem} cannot be applied since the implied constant in the theorem depends on $d$. This is because the constant $M_\alpha$ in Assumption \ref{EMB} and the implied constants in Assumption \ref{EVD} depend on the dimension $d$ in general, see Example \ref{example: dot-product} below for instance. Indeed, in high dimension, Kernel interpolation can generalize well under certain conditions \citep{liang2020just}. Similar results for benign interpolations have been established for several models, such as linear models \citep{bartlett2020benign,hastie2022surprises}, random features models \citep{mei2022generalization} and neural networks \citep{cao2022benign}. These results provide some explanations for the empirical observation that deep neural networks often interpolate the training data and yet generalize well \citep{zhang2017understanding}. On the contrary, for fixed input dimension, Theorem \ref{main theorem} shows that the kernel interpolation is inconsistent in the $[\cH]_\nu^\gamma$-norm for $\gamma>3(\alpha^* -1/\beta)$, and has poor generalization for $\gamma = 3(\alpha^* -1/\beta)$.

\begin{corollary}\label{main corollary}
Under the assumption of Theorem \ref{main theorem}, we have
\begin{enumerate}[label=\textnormal{(\arabic*)}]
\item If $\gamma>3(\alpha^* -1/\beta)$, the kernel interpolation is inconsistent in $[\cH]_\nu^\gamma$-norm: for any $\delta\in (0,1)$, with probability at least $1-\delta$,
\[
\bE \left[ \left\|[\widehat{f}_n]_\nu - f_P^*\right\|_\gamma^2 |X \right] \to \infty, \quad \mbox{as } n\to \infty.
\]

\item If $\gamma=3(\alpha^* -1/\beta)$, the kernel interpolation generalizes poorly in $[\cH]_\nu^\gamma$-norm: for any $\epsilon>0$ and $\delta\in (0,1)$, when $n$ is sufficiently large (depending on $\epsilon$ and $\delta$), it holds with probability at least $1-\delta$ that
\[
\bE \left[ \left\|[\widehat{f}_n]_\nu - f_P^*\right\|_\gamma^2 |X \right] \gtrsim \sigma^2 n^{-\epsilon}.
\]
If $\alpha^*$ also satisfies Assumption \ref{EMB}, i.e. $M_{\alpha^*}<\infty$, then the lower bound can be improved to
\[
\bE \left[ \left\|[\widehat{f}_n]_\nu - f_P^*\right\|_\gamma^2 |X \right] \gtrsim \sigma^2 (\log n)^{-1-2\zeta-\epsilon},
\]
which implies that the kernel interpolation is inconsistent when $\zeta <-1/2$.
\end{enumerate}
\end{corollary}

We illustrate the result by several examples with $\alpha^*=1/\beta$.

\begin{example}[Kernels with uniformly bounded eigenfunctions]
Assume that the eigenfunctions of $T_\nu$ are uniformly bounded, i.e. $\sup_{i\in\bN}\|e_i\|_{L^\infty(\nu)} \le B<\infty$, then, for any $\alpha\in (1/\beta,1]$, 
\[
M_\alpha^2= \left\| \sum_{i=1}^\infty \mu_i^\alpha e_i^2 \right\|_{L^\infty(\nu)} \le B^2 \sum_{i=1}^\infty \mu_i^\alpha \asymp \sum_{i=1}^\infty (i(\log i)^\zeta)^{-\alpha \beta} <\infty,
\]
by equality (\ref{emb norm}) and Assumption \ref{EVD}. Hence, the embedding index is $\alpha^*=1/\beta$. If $\zeta>1$, then we further have $M_{\alpha^*}<\infty$. Corollary \ref{main corollary} implies that the kernel interpolation is inconsistent in $[\cH]_\nu^\gamma$-norm for any $\gamma>0$, and it generalizes poorly in $L^2(\nu)$-norm, which was also proven by \cite{li2024kernel} with additional smoothness assumption on the kernel.  However, we remark that even $C^\infty$-kernels on $[0,1]$, which is equipped with the Lebesgue measure, may not have uniformly bounded eigenfunctions \citep{zhou2002covering}.
\end{example}

\begin{example}[Besov RKHSs]\label{example besov}
Let us recall some basic properties of Besov RKHSs (see \citet[Section 4]{fischer2020sobolev} for more details). Let $\cX\subseteq \bR^d$ be a bounded, connected, and open domain with smooth boundary (so that the Sobolev embedding theorem can be applied below). Let $\nu$ be a probability measure whose Radon-Nikodym derivative $d\nu/d\mu$ with respect to the Lebesgue measure $\mu$ satisfies $c\le d\nu/d\mu \le C$ for some constants $c,C>0$. For $r>0$, one can define the Besov space $\cB^r_{2,2}(\cX)$ as the real interpolation space
\[
\cB^r_{2,2}(\cX) := [L^2(\cX), \cW^{m,2}(\cX)]_{\frac{r}{m},2},
\]
where $\cW^{m,2}$ is the Sobolev space of smoothness $m=\min\{k\in \bN:k>r \}$. When $r>d/2$, the Sobolev embedding theorem for Besov spaces shows that $\cB^r_{2,2}(\cX) \hookrightarrow L^\infty(\cX)$ and each $\mu$-equivalence class in $\cB^r_{2,2}(\cX)$ has a unique bounded and continuous representative \citep[Theorem 7.34]{adams2003sobolev}. Hence, one can define the Besov RKHS as $H^r(\cX) := \{f\in C_0(\cX):[f]_\mu \in \cB^r_{2,2}(\cX) \}$, equipped with the norm $\|f\|_{H^r(\cX)}:=\|[f]_\mu\|_{\cB^r_{2,2}(\cX)}$. It is well-known that the space $H^r(\cX)$ is a separable RKHS with bounded measurable kernel and the eigenvalues of $T_\nu$ satisfy Assumption \ref{EVD} with $\beta = 2r/d$, since $L^2(\nu) \cong L^2(\cX)$. By the interpolation property of Besov spaces, for any $\alpha\in (1/\beta,1)$, 
\begin{align*}
[H^r(\cX)]_\nu^\alpha &\cong [L^2(\nu),[H^r(\cX)]^1_\nu]_{\alpha,2} \cong [L^2(\cX),[H^r(\cX)]^1_\mu]_{\alpha,2} \\
&\cong \cB^{r\alpha}_{2,2}(\cX) \hookrightarrow L^\infty(\cX) \cong L^\infty(\nu),
\end{align*}
since $r\alpha> r/\beta = d/2$. As a consequence, the embedding index of the RKHS $H^r(\cX)$ with $r>d/2$ is $\alpha^*=1/\beta$. Hence, Corollary \ref{main corollary} shows that the kernel interpolation is inconsistent in $\cB^s_{2,2}(\cX)$-norm for any $s>0$ and generalizes poorly in $L^2(\nu)$-norm.
\end{example}

\begin{example}[Dot-product kernels]\label{example: dot-product}
Let $K(x,x')=g(\langle x,x' \rangle)$ be a dot-product kernel on $\cX=\bS^d$, i.e. the unit sphere of $\bR^{d+1}$. Assume that $\nu$ is the uniform measure on $\bS^d$. It is well-known that the eigenfunctions of $T_\nu$ are spherical harmonics and we have the following Mercer's decomposition \citep[Section 17.2]{wendland2004scattered}
\[
K(x,x') = \sum_{k=0}^\infty a_k \sum_{j=1}^{N(d,k)} Y_{k,j}(x) Y_{k,j}(x'),
\]
where $N(d,k):=\binom{k+d}{k} -\binom{k-2+d}{k-2}$ is the dimension of the spherical harmonic space $\bY_k$ of degree $k$, which contains the restrictions of real harmonic homogeneous polynomials of degree $k$ to $\bS^d$, and $\{Y_{k,j}:1\le j\le N(d,k)\}$ is any orthonormal basis of $\bY_k$. The addition formula \cite[Theorem 1.2.6]{dai2013approximation} shows
\[
\sum_{j=1}^{N(d,k)} Y_{k,j}(x) Y_{k,j}(x') = N(d,k) P_i(\langle x,x' \rangle), \quad x,x' \in \bS^d,
\]
where $P_k$ is the Gegenbauer polynomial of degree $k$ with normalization $P_k(1)=1$. Note that the eigenvalues $\{a_k\}_{k=0}^\infty$ is not necessary in decreasing order. However, one can show that $a_k \asymp k^{-d\beta} (\log k)^{-\zeta\beta}$ for large $k$ implies Assumption \ref{EVD} by using the fact that $N(d,k)\asymp k^{d-1}$ and $\sum_{r=0}^k N(d,r)\asymp k^d$. Thus, in this case, for any $\alpha\in 1/\beta,1]$,
\begin{align*}
M_\alpha^2 &= \sup_{x\in \bS^d} \sum_{k=0}^\infty a_k^\alpha \sum_{j=1}^{N(d,k)} Y_{k,j}(x)^2 = \sum_{k=0}^\infty a_k^\alpha N(d,k) \\
&\asymp \sum_{k=0}^\infty k^{-d\alpha \beta + d -1} (\log k)^{-\zeta\alpha\beta} <\infty,
\end{align*}
which shows that $\alpha^* =1/\beta$. Furthermore, if $\zeta>1$, then we have $M_{\alpha^*}<\infty$. 
\end{example}

We note that the above results for dot-product kernels can be used to study neural networks by considering the neural tangent kernel \citep{jacot2018neural,bietti2019inductive}:
\[
K(x,x') = \frac{2}{\pi} \langle x,x' \rangle \left(\pi - \arccos \langle x,x' \rangle\right) + \frac{1}{\pi} \sqrt{1- \langle x,x' \rangle^2}.
\]
This kernel is derived from lazy training of shallow ReLU neural networks. It is known that it satisfies Assumption \ref{EVD} with $\beta=1+1/d$ and $\zeta=0$ \citep{bietti2019inductive}. However, beyond lazy training, the function space corresponding to neural network learning is a reproducing kernel Banach space (RKBS, \citep{lin2022reproducing,zhang2009reproducing}), rather than RKHS \citep{bartolucci2023understanding,parhi2023minimax,yang2024nonparametric}. One of the interesting research direction is to generalize the above analysis for RKHSs to RKBSs. The main difficulty is that we are not able to apply integral operator techniques for RKBSs. Hence, a different approach is needed.

\section{Proofs of technical results}\label{sec: proofs}

This section give the proofs of Lemmas \ref{lower bound by variance}, \ref{v bound 1}, \ref{v bound 2} and \ref{lower bound v2}.

\subsection{Proof of Lemma \ref{lower bound by variance}}\label{sec: lower bound by variance}

Recall that $([e_j]_\nu)_{j\in\bN}$ forms an orthonormal basis of $[\cH]_\nu^0$. Let us denote $a_j = \langle [f]_\nu, [e_j]_\nu \rangle_{L^2(\nu)}$ and $b_j = \langle f_P^*, [e_j]_\nu \rangle_{L^2(\nu)}$. 
Then, the Parseval's identity with respect to the orthonormal basis $(\mu_j^{\gamma/2}[e_j]_\nu)_{j\in\bN}$ of $[\cH]_\nu^\gamma$ shows
\begin{align*}
\bE \left[ \|[f]_\nu - f_P^*\|_\gamma^2 |X \right] &= \bE \left[ \sum_{j=1}^\infty \mu_j^{-\gamma} (a_j-b_j)^2 |X \right] \\
&= \sum_{j=1}^\infty \mu_j^{-\gamma} \bE \left[ (a_j-\bE[a_j|X])^2 + (\bE[a_j|X]-b_j)^2 |X \right] \\
&\ge \sum_{j=1}^\infty \mu_j^{-\gamma} \bE \left[ (a_j-\bE[a_j|X])^2 |X \right].
\end{align*}
By Assumption \ref{ass noise}, we have the following bound almost surely
\begin{align*}
\bE \left[ (a_j-\bE[a_j|X])^2 |X \right] &= \bE \left[ \left(\frac{1}{n}\sum_{i=1}^n (y_i- f_P^*(x_i)) \langle [u_i]_\nu, [e_j]_\nu \rangle_{L^2(\nu)}\right)^2 |X \right] \\
&\ge \frac{\sigma^2}{n^2} \sum_{i=1}^n \langle [u_i]_\nu, [e_j]_\nu \rangle_{L^2(\nu)}^2.
\end{align*}
As a consequence, 
\begin{align*}
\bE \left[ \|[f]_\nu - f_P^*\|_\gamma^2 |X \right] &\ge \frac{\sigma^2}{n^2} \sum_{i=1}^n \sum_{j=1}^\infty \mu_j^{-\gamma} \langle [u_i]_\nu, [e_j]_\nu \rangle_{L^2(\nu)}^2 \\
&= \frac{\sigma^2}{n^2} \sum_{i=1}^n \|[u_i]_\nu\|_\gamma^2,
\end{align*}
where we use the Parseval's identity in the last equality.

\subsection{Proof of Lemma \ref{v bound 1}}\label{sec: v bound 1}

Let us first estimate the difference between the quantities $\|C_\nu^{(1-\gamma)/2} (C_{\nu_n} + \lambda)^{-1} K_x \|_\cH$ and $\|C_\nu^{(1-\gamma)/2} (C_\nu + \lambda)^{-1} K_x \|_\cH$ for any $x\in \cX$. Notice that
\[
(C_{\nu_n} + \lambda)^{-1} - (C_\nu + \lambda)^{-1} = (C_\nu + \lambda)^{-1} (C_\nu - C_{\nu_n}) (C_{\nu_n} + \lambda)^{-1}.
\]
We have the following bound
\begin{equation}\label{ineq 1}
\begin{aligned}
&\left| \left\|C_\nu^{(1-\gamma)/2} (C_{\nu_n} + \lambda)^{-1} K_x \right\|_\cH - \left\|C_\nu^{(1-\gamma)/2} (C_\nu + \lambda)^{-1} K_x \right\|_\cH \right| \\
\le & \left\|C_\nu^{(1-\gamma)/2} (C_\nu + \lambda)^{-1} (C_\nu - C_{\nu_n}) (C_{\nu_n} + \lambda)^{-1} K_x \right\|_\cH \\
\le & \left\|C_\nu^{(1-\gamma)/2} (C_\nu + \lambda)^{-1/2}  \right\| \left\|(C_\nu + \lambda)^{-1/2} (C_\nu - C_{\nu_n}) (C_\nu + \lambda)^{-1/2}\right\| \\
& \qquad \times \left\|(C_\nu + \lambda)^{1/2} (C_{\nu_n} + \lambda)^{-1} (C_\nu + \lambda)^{1/2}\right\| \left\|(C_\nu + \lambda)^{-1/2} K_x\right\|_{\cH}.
\end{aligned}
\end{equation}
Thus, we need to estimate the norms in the last expression.

Firstly, by the spectral decomposition (\ref{spectral decomposition}), if we fix an arbitrary orthonormal basis $(\widetilde{e}_j)_{j\in \bJ} \subseteq \cH$ of $\ker I_\nu$ with $\bJ \cap \bN=\emptyset$, then we have the spectral representation 
\begin{equation}\label{spectral decomposition 2}
(C_\nu + \lambda)^{-s} = \sum_{i=1}^\infty (\mu_i + \lambda)^{-s} \langle \cdot,\sqrt{\mu_i} e_i \rangle_\cH \sqrt{\mu_i} e_i + \lambda^{-s} \sum_{j\in \bJ} \langle \cdot, \widetilde{e}_j \rangle_\cH \widetilde{e}_j, \quad s>0.
\end{equation}
As a consequence, using $s=1/2$ and the decomposition (\ref{spectral decomposition}),
\begin{equation}\label{ineq 2}
\left\|C_\nu^{(1-\gamma)/2} (C_\nu + \lambda)^{-1/2}  \right\| = \sup_{i\in\bN} \left( \frac{\mu_i^{1-\gamma}}{\mu_i + \lambda} \right)^{1/2} \le \lambda^{-\gamma/2},
\end{equation}
where the last inequality follows from Proposition \ref{appendix bound 1}. On the other hand, by the reproducing property of the kernel $K$, for all $x\in \cX$,
\begin{align*}
K_x &= \sum_{i=1}^\infty \langle K_x,\sqrt{\mu_i} e_i \rangle_\cH \sqrt{\mu_i} e_i + \sum_{j\in \bJ} \langle K_x, \widetilde{e}_j \rangle_\cH \widetilde{e}_j \\
&= \sum_{i=1}^\infty \mu_i e_i(x)e_i + \sum_{j\in \bJ} \widetilde{e}_j(x) \widetilde{e}_j,
\end{align*}
which converges in $\cH$. Since $\widetilde{e}_j \in \ker I_\nu$, the following Mercer representation \citep{steinwart2012mercer} holds for $\nu$-almost all $x\in \cX$,
\begin{equation}\label{Mercer}
K_x = \sum_{i=1}^\infty \mu_i e_i(x)e_i.
\end{equation}
Thus, for $s\ge (1-\alpha)/2$ and $\nu$-almost all $x\in \cX$,
\begin{equation}\label{ineq 3}
\begin{aligned}
\left\|(C_\nu + \lambda)^{-s} K_x\right\|_{\cH}^2 
&= \sum_{i=1}^\infty (\mu_i+\lambda)^{-2s} \mu_i e_i^2(x) \\
&\le \sum_{i=1}^\infty \lambda^{1-2s-\alpha} \mu_i^\alpha e_i^2(x) \\
&\le M_\alpha^2 \lambda^{1-2s-\alpha},
\end{aligned}
\end{equation}
where we apply the bound $\sup_{t\ge 0} t^p/(t+\lambda) \le \lambda^{p-1}$ for $p=(1-\alpha)/(2s) \in [0,1]$ (see Proposition \ref{appendix bound 1} in the Appendix) in the first inequality, and use (\ref{emb norm}) in the last inequality. 

Next, we apply Lemma \ref{operator concentration}, which is from \citet[Lemma 17]{fischer2020sobolev}, to obtain that the following inequality holds with probability at least $1-2e^{-\tau}$,
\[
\left\|(C_\nu + \lambda)^{-1/2} (C_\nu - C_{\nu_n}) (C_\nu + \lambda)^{-1/2}\right\| \le \frac{4M_\alpha^2 \tau B_\nu(\lambda)}{3n\lambda^\alpha} + \sqrt{\frac{2M_\alpha^2 \tau B_\nu(\lambda)}{n\lambda^\alpha}},
\]
where $N_\nu(\lambda)$ is the effective dimension (\ref{effective dim}) and
\[
B_\nu(\lambda) = \log\left(2e N_\nu(\lambda) \frac{\|C_\nu\|+\lambda}{\|C_\nu\|}\right).
\]
By Assumption \ref{EMB}, we get the estimate (\ref{effective dim bound}), i.e. $N_\nu(\lambda) \lesssim \lambda^{-\alpha}$, which implies that $B_\nu(\lambda) \lesssim \log \lambda^{-1}$. Consequently, there exists a constant $c>0$ such that, if $\lambda^{-\alpha} \log \lambda^{-1} \le cn \kappa^2 \tau^{-1}$, then $M_\alpha^2 \tau B_\nu(\lambda)/(n\lambda^\alpha) \le \kappa^2/128$ and hence (recall that $\kappa\le 1$)
\begin{equation}\label{ineq 4}
\left\|(C_\nu + \lambda)^{-1/2} (C_\nu - C_{\nu_n}) (C_\nu + \lambda)^{-1/2}\right\| \le \frac{\kappa^2}{96} + \frac{\kappa}{8} < \frac{\kappa}{4}.
\end{equation}

Using the identity 
\[
C_{\nu_n} + \lambda = (C_\nu +\lambda)^{1/2}\left(1 - (C_\nu + \lambda)^{-1/2} (C_\nu - C_{\nu_n}) (C_\nu + \lambda)^{-1/2}\right) (C_\nu +\lambda)^{1/2},
\]
one can see that 
\[
\left\|(C_\nu + \lambda)^{1/2} (C_{\nu_n} + \lambda)^{-1} (C_\nu + \lambda)^{1/2} \right\| =  \left\| \left(1 - (C_\nu + \lambda)^{-1/2} (C_\nu - C_{\nu_n}) (C_\nu + \lambda)^{-1/2}\right)^{-1} \right\|.
\]
By inequality (\ref{ineq 4}), the operator in the last norm can be represented by the Neumann series. This gives us the following bound
\begin{equation}\label{ineq 5}
\begin{aligned}
\left\|(C_\nu + \lambda)^{1/2} (C_{\nu_n} + \lambda)^{-1} (C_\nu + \lambda)^{1/2} \right\| &\le \sum_{k=0}^\infty \left\|(C_\nu + \lambda)^{-1/2} (C_\nu - C_{\nu_n}) (C_\nu + \lambda)^{-1/2}\right\|^k \\
&\le \sum_{k=0}^\infty \frac{1}{4^k} = \frac{4}{3}.
\end{aligned}
\end{equation}

Combining inequalities (\ref{ineq 1}) (\ref{ineq 2}) (\ref{ineq 3}) (\ref{ineq 4}) and (\ref{ineq 5}), we have
\[
\left| \left\|C_\nu^{(1-\gamma)/2} (C_{\nu_n} + \lambda)^{-1} K_x \right\|_\cH - \left\|C_\nu^{(1-\gamma)/2} (C_\nu + \lambda)^{-1} K_x \right\|_\cH \right| \le \frac{1}{3} \kappa M_\alpha \lambda^{-(\gamma+\alpha)/2},
\]
holds with probability at least $1-2e^{-\tau}$, if $\lambda^{-\alpha} \log \lambda^{-1} \le cn \kappa^2 \tau^{-1}$. Using inequalities (\ref{ineq 2}) and (\ref{ineq 3}) again, we get
\begin{equation}\label{ineq 6}
\begin{aligned}
\|C_\nu^{(1-\gamma)/2} (C_\nu + \lambda)^{-1} K_x \|_\cH &\le \left\|C_\nu^{(1-\gamma)/2} (C_\nu + \lambda)^{-1/2}  \right\| \left\|(C_\nu + \lambda)^{-1/2} K_x\right\|_{\cH} \\
&\le M_\alpha \lambda^{-(\gamma+\alpha)/2}.
\end{aligned}
\end{equation}
Consequently,
\[
\left\|C_\nu^{(1-\gamma)/2} (C_{\nu_n} + \lambda)^{-1} K_x \right\|_\cH + \left\|C_\nu^{(1-\gamma)/2} (C_\nu + \lambda)^{-1} K_x \right\|_\cH \le 3 M_\alpha \lambda^{-(\gamma+\alpha)/2}.
\]
Finally, by the definition of $V(\lambda)$ and $V_1(\lambda)$, 
\begin{align*}
|V(\lambda) - V_1(\lambda)| &\le \frac{1}{n^2} \sum_{i=1}^n \left| \left\|C_\nu^{(1-\gamma)/2} (C_{\nu_n} + \lambda)^{-1} K_{x_i} \right\|_\cH^2 - \left\|C_\nu^{(1-\gamma)/2} (C_\nu + \lambda)^{-1} K_{x_i} \right\|_\cH^2 \right| \\
&\le \frac{1}{n^2} \sum_{i=1}^n \kappa M_\alpha^2 \lambda^{-\gamma-\alpha} = \frac{\kappa M_\alpha^2}{n \lambda^{\gamma+\alpha}},
\end{align*}
which completes the proof.

\subsection{Proof of Lemma \ref{v bound 2}}\label{sec: v bound 2}

By the Mercer representation (\ref{Mercer}) and spectral decompositions (\ref{spectral decomposition}) and (\ref{spectral decomposition 2}), for $\nu$-almost all $x\in \cX$,
\[
\|C_\nu^{(1-\gamma)/2} (C_\nu + \lambda)^{-1} K_x \|_\cH^2 = \sum_{i=1}^\infty \frac{\mu_i^{1-\gamma}}{(\mu_i+\lambda)^2} \langle K_x,\sqrt{\mu_i} e_i \rangle_\cH^2 = \sum_{i=1}^\infty \frac{\mu_i^{2-\gamma}}{(\mu_i+\lambda)^2} e_i^2(x).
\]
Since $\|e_i\|_{L^2(\nu)}=1$,
\begin{align*}
nV_2(\lambda) &= \int_\cX \|C_\nu^{(1-\gamma)/2} (C_\nu + \lambda)^{-1} K_x \|_\cH^2 d\nu(x) \\
&= \sum_{i=1}^\infty \frac{\mu_i^{2-\gamma}}{(\mu_i+\lambda)^2}.
\end{align*}

If Assumption \ref{EMB} holds, we consider the random variable $\xi_i := \|C_\nu^{(1-\gamma)/2} (C_\nu + \lambda)^{-1} K_{x_i} \|_\cH^2$, which has mean $\bE[\xi_i] = nV_2(\lambda)$. By inequality (\ref{ineq 6}), $\xi_i \in [0,M_\alpha^2 \lambda^{-\gamma-\alpha}]$ and hence it is sub-Gaussian with parameter $M_\alpha^2 \lambda^{-\gamma-\alpha}/2$. Applying Hoeffding inequality (see Proposition \ref{hoeffding} in the Appendix) to $\xi_i$ and $-\xi_i$, we get, for all $t\ge 0$,
\[
\bP\left[ \left|\sum_{i=1}^n (\xi_i - nV_2(\lambda))\right| \ge t \right] \le 2 \exp\left( -\frac{2t^2 \lambda^{2(\gamma+\alpha)}}{nM_\alpha^4} \right).
\]
For any $\tau\ge 1$, taking $t= \sqrt{\tau n} M_\alpha^2 \lambda^{-\gamma-\alpha}/\sqrt{2}$ implies that
\[
n^2 |V_1(\lambda)- V_2(\lambda)| = \left|\sum_{i=1}^n (\xi_i - nV_2(\lambda))\right| \le \frac{\sqrt{\tau n}M_\alpha^2}{\sqrt{2} \lambda^{\gamma+\alpha}},
\]
holds with probability at least $1-2e^{-\tau}$.

\subsection{Proof of Lemma \ref{lower bound v2}}\label{sec: lower bound v2}

let us denote $\phi(x) = x(\log x)^\zeta$ for $x\ge 1$. It is easy to check that $\phi$ is increasing on the interval $(\max\{e^{-\zeta}, 1\},\infty)$. Hence, it is invertible on this interval, and we denote the inverse by $\phi^{-1}$. By Assumption \ref{EVD}, there exist constants $c_1,c_2>0$ such that $c_1 \phi(i)^{-\beta} \le \mu_i \le c_2 \phi(i)^{-\beta}$ for $i\ge 2$. We observe that the function $t \mapsto t^{2-\gamma}/(t+\lambda)^2$ is increasing on the interval $(0,\lambda)\subseteq (0,(2-\gamma)\lambda/\gamma)$, see the proof of Proposition \ref{appendix bound 1} in the Appendix for example. Thus,
\begin{align*}
\sum_{i=1}^\infty \frac{\mu_i^{2-\gamma}}{(\mu_i+\lambda)^2} &\ge \sum_{c_2 \phi(i)^{-\beta} \le \lambda} \frac{\mu_i^{2-\gamma}}{(\mu_i+\lambda)^2} \\
&\ge \sum_{c_2 \phi(i)^{-\beta} \le \lambda} \frac{c_1^{2-\gamma}\phi(i)^{-\beta(2-\gamma)}}{(c_1 \phi(i)^{-\beta}+\lambda)^2}\\
&= c_1^{2-\gamma} \sum_{\lambda \phi(i)^\beta \ge c_2} \frac{ \phi(i)^{\beta \gamma}}{(c_1 +\lambda \phi(i)^{\beta})^2} \\
&\ge c_1^{2-\gamma} c_2^\gamma \lambda^{-\gamma} \sum_{\lambda \phi(i)^\beta \ge c_2} \frac{1}{(c_1 +\lambda \phi(i)^{\beta})^2}.
\end{align*}
It remains to lower bound the last sum. For a sufficiently large constant $c_3\ge c_2$, we have 
\[
\sum_{\lambda \phi(i)^\beta \ge c_2} \frac{1}{(c_1 +\lambda \phi(i)^{\beta})^2} \ge \int_{\lambda \phi(x)^\beta \ge c_3} \frac{1}{(c_1 +\lambda \phi(x)^{\beta})^2} dx =: \cI,
\]
and $\phi(x)$ is increasing for $x$ satisfying $\lambda \phi(x)^\beta \ge c_3$. We are going to make a change of variable $t = \lambda^{1/\beta}\phi(x)$, then $x = \phi^{-1}(\lambda^{-1/\beta}t)$ and 
\[
dt = \lambda^{1/\beta} \phi'(x) dx = \lambda^{1/\beta} \phi'(\phi^{-1}(\lambda^{-1/\beta}t)) dx.
\]
As a consequence,
\[
\cI = \lambda^{-1/\beta} \int_{c_3^{1/\beta}}^\infty \frac{1}{(c_1 +t^{\beta})^2} \frac{1}{\phi'(\phi^{-1}(\lambda^{-1/\beta}t))} dt.
\]
Notice that $\phi'(x) = (\log x)^\zeta + \zeta (\log x)^{\zeta-1} \le (1+|\zeta|)(\log x)^\zeta$ for $x\ge e$. Furthermore, the equality $t=\lambda^{1/\beta}\phi(x) = \lambda^{1/\beta} x(\log x)^\zeta$ implies that $\log x \asymp \log \lambda^{-1} + \log t$. Thus,
\[
\frac{1}{\phi'(\phi^{-1}(\lambda^{-1/\beta}t))} \gtrsim \frac{1}{(\log \lambda^{-1} + \log t)^\zeta} = \frac{(\log \lambda^{-1})^{-\zeta}}{(1+\log t/ \log \lambda^{-1})^\zeta} .
\]
Therefore, when $\zeta \le 0$,
\[
\cI \gtrsim \lambda^{-1/\beta}(\log \lambda^{-1})^{-\zeta} \int_{c_3^{1/\beta}}^\infty \frac{1}{(c_1 +t^{\beta})^2} dt \gtrsim \lambda^{-1/\beta}(\log \lambda^{-1})^{-\zeta},
\]
since $\beta>1$. When $\zeta >0$, since $0<\lambda<1/2$, 
\[
\cI \gtrsim \lambda^{-1/\beta}(\log \lambda^{-1})^{-\zeta} \int_{c_3^{1/\beta}}^\infty \frac{1}{(c_1 +t^{\beta})^2} \frac{1}{(1+\log t/ \log 2)^\zeta} dt \gtrsim \lambda^{-1/\beta}(\log \lambda^{-1})^{-\zeta},
\]
which gives the desired lower bound.

\section*{Acknowledgments}

The work described in this paper was partially supported by National Natural Science Foundation of China under Grant 12371103. We thank the referees for their helpful comments and suggestions on the paper.

\appendix

\section{Useful inequalities}

The following two propositions are useful for estimating the norm of some integral operators in RKHSs.

\begin{proposition}\label{appendix bound 1}
Let $\lambda>0$, $0\le \gamma\le 1$ and $f_\gamma(t) = \frac{t^{\gamma}}{t+\lambda}$, then $f_\gamma(t) \le \lambda^{\gamma-1}$ for $t\ge 0$.
\end{proposition}
\begin{proof}
If $\gamma=0$, then $f_\gamma(t)$ is decreasing. If $\gamma=1$, then $f_\gamma(t)$ is increasing. If $0<\gamma<1$, $f_\gamma'(t) = t^{\gamma-1}(t+\lambda)^{-2}(\gamma\lambda - (1-\gamma)t)$, which implies $f_\gamma(t)$ is increasing on $[0,\frac{\gamma \lambda}{1-\gamma}]$ and decreasing on $[\frac{\gamma \lambda}{1-\gamma},\infty)$. The bound then follows from direct calculation.
\end{proof}

\begin{proposition}\label{appendix bound 2}
If $\lambda>0$ and $(\mu_i)_{i=1}^\infty$ is a non-increasing sequence of positive numbers which satisfies $\sum_{i=1}^\infty \mu_i^\alpha<\infty$ for some $\alpha\in (0,1]$, then
\[
\sum_{i=1}^\infty \frac{\mu_i}{\mu_i+\lambda}\lesssim \lambda^{-\alpha}.
\]
\end{proposition}
\begin{proof}
When $\alpha=1$, we have
\[
\sum_{i=1}^\infty \frac{\mu_i}{\mu_i+\lambda} \le \frac{1}{\lambda} \sum_{i=1}^\infty \mu_i \lesssim \lambda^{-1}.
\]
When $\alpha\in (0,1)$, we notice that
\[
i \mu_i^\alpha \le \sum_{j=1}^i \mu_j^\alpha \le \sum_{j=1}^\infty \mu_j^\alpha <\infty,
\]
which shows that $\mu_i \le ci^{-1/\alpha}$ for some constant $c>0$. Since the function $t\mapsto t/(t+\lambda)$ is increasing, we have 
\begin{align*}
\sum_{i=1}^\infty \frac{\mu_i}{\mu_i+\lambda} &\le \sum_{i=1}^\infty \frac{c i^{-1/\alpha}}{c i^{-1/\alpha}+\lambda} = \sum_{i=1}^\infty \frac{c}{c+\lambda i^{1/\alpha}} \le \int_0^\infty \frac{c}{c+\lambda x^{1/\alpha}}dx \\
&= \lambda^{-\alpha} \int_0^\infty \frac{c}{c+t^{1/\alpha}}dt \lesssim \lambda^{-\alpha},
\end{align*}
which completes the proof.
\end{proof}

The following form of Hoeffding inequality is from \citet[Proposition 2.5]{wainwright2019high}.

\begin{proposition}[Hoeffding inequality]\label{hoeffding}
Suppose that the random variables $\xi_i$, $i=1,\dots,n$, are independent and sub-Gaussian with parameter $\sigma_i$, i.e. $\bE[\exp(\lambda (\xi_i-\bE[\xi_i]))] \le \exp(\sigma_i^2\lambda^2/2)$ for all $\lambda\in\bR$. Then, for all $t\ge 0$,
\[
\bP\left[ \sum_{i=1}^n (\xi_i - \bE[\xi_i]) \ge t \right] \le \exp\left( -\frac{t^2}{2\sum_{i=1}^n \sigma_i^2} \right).
\]
\end{proposition}

We also need some concentration bounds for operators. The next two lemmas are from \citet[Lemma 12 and Lemma 17]{fischer2020sobolev}.

\begin{lemma}\label{norm eq}
Let $\cH$ be a separable RKHS on a measurable space $\cX$ with a bounded measurable kernel $K$, and $\nu$ be a probability distribution on $\cX$. Then, for any $0\le \gamma \le 1$ and $f\in \cH$, it holds
\[
\|[f]_\nu\|_\gamma \le \| C_\nu^{(1-\gamma)/2} f\|_\cH.
\]
The equality holds if $\gamma<1$ or $f\in (\ker I_\nu)^\perp$.
\end{lemma}

\begin{lemma}\label{operator concentration}
Let $\cH$ be a separable RKHS on a measurable space $\cX$ with  a bounded measurable kernel $K$, and $\nu$ be a probability distribution on $\cX$. If Assumption \ref{EMB} holds, then for $\tau \ge 1$, $\lambda>0$ and $n\ge 1$, with probability at least $1-2e^{-\tau}$ over $n$ independent samples $\{x_i\}_{i=1}^n$ from $\nu$, the following operator norm bound is satisfied
\[
\left\|(C_\nu + \lambda)^{-1/2} (C_\nu - C_{\nu_n}) (C_\nu + \lambda)^{-1/2}\right\| \le \frac{4M_\alpha^2 \tau B_\nu(\lambda)}{3n\lambda^\alpha} + \sqrt{\frac{2M_\alpha^2 \tau B_\nu(\lambda)}{n\lambda^\alpha}},
\]
where $\nu_n = \frac{1}{n}\sum_{i=1}^n \delta_{x_i}$ is the empirical distribution and
\begin{align*}
B_\nu(\lambda) &:= \log\left(2e N_\nu(\lambda) \frac{\|C_\nu\|+\lambda}{\|C_\nu\|}\right),\\
N_\nu(\lambda) &:= \tr((C_\nu+\lambda)^{-1}C_\nu) = \sum_{i=1}^\infty \frac{\mu_i}{\mu_i+\lambda}.
\end{align*}
\end{lemma}

\bibliographystyle{myplainnat}
\bibliography{references}

\begin{thebibliography}{37}
\providecommand{\natexlab}[1]{#1}
\providecommand{\url}[1]{\texttt{#1}}
\expandafter\ifx\csname urlstyle\endcsname\relax
  \providecommand{\doi}[1]{doi: #1}\else
  \providecommand{\doi}{doi: \begingroup \urlstyle{rm}\Url}\fi

\bibitem[Adams and Fournier(2003)]{adams2003sobolev}
Robert~A. Adams and John~J.F. Fournier.
\newblock \emph{Sobolev Spaces}.
\newblock Elsevier, 2003.

\bibitem[Bartlett et~al.(2020)Bartlett, Long, Lugosi, and
  Tsigler]{bartlett2020benign}
Peter~L. Bartlett, Philip~M. Long, G{\'{a}}bor Lugosi, and Alexander Tsigler.
\newblock Benign overfitting in linear regression.
\newblock \emph{Proceedings of the National Academy of Sciences}, 117\penalty0
  (48):\penalty0 30063--30070, 2020.

\bibitem[Bartlett et~al.(2021)Bartlett, Montanari, and
  Rakhlin]{bartlett2021deep}
Peter~L. Bartlett, Andrea Montanari, and Alexander Rakhlin.
\newblock Deep learning: a statistical viewpoint.
\newblock \emph{Acta Numerica}, 30:\penalty0 87--201, 2021.

\bibitem[Bartolucci et~al.(2023)Bartolucci, Vito, Rosasco, and
  Vigogna]{bartolucci2023understanding}
Francesca Bartolucci, Ernesto~De Vito, Lorenzo Rosasco, and Stefano Vigogna.
\newblock Understanding neural networks with reproducing kernel {Banach}
  spaces.
\newblock \emph{Applied and Computational Harmonic Analysis}, 62:\penalty0
  194--236, 2023.

\bibitem[Beaglehole et~al.(2023)Beaglehole, Belkin, and
  Pandit]{beaglehole2023inconsistency}
Daniel Beaglehole, Mikhail Belkin, and Parthe Pandit.
\newblock On the inconsistency of kernel ridgeless regression in fixed
  dimensions.
\newblock \emph{SIAM Journal on Mathematics of Data Science}, 5\penalty0
  (4):\penalty0 854--872, 2023.

\bibitem[Belkin(2021)]{belkin2021fit}
Mikhail Belkin.
\newblock Fit without fear: remarkable mathematical phenomena of deep learning
  through the prism of interpolation.
\newblock \emph{Acta Numerica}, 30:\penalty0 203--248, 2021.

\bibitem[Belkin et~al.(2019)Belkin, Hsu, Ma, and Mandal]{belkin2019reconciling}
Mikhail Belkin, Daniel Hsu, Siyuan Ma, and Soumik Mandal.
\newblock Reconciling modern machine-learning practice and the classical
  bias{\textendash}variance trade-off.
\newblock \emph{Proceedings of the National Academy of Sciences}, 116\penalty0
  (32):\penalty0 15849--15854, 2019.

\bibitem[Bietti and Mairal(2019)]{bietti2019inductive}
Alberto Bietti and Julien Mairal.
\newblock On the inductive bias of neural tangent kernels.
\newblock In \emph{Advances in Neural Information Processing Systems}, pages
  12873--12884. 2019.

\bibitem[Blanchard and M{\"u}cke(2018)]{blanchard2018optimal}
Gilles Blanchard and Nicole M{\"u}cke.
\newblock Optimal rates for regularization of statistical inverse learning
  problems.
\newblock \emph{Foundations of Computational Mathematics}, 18\penalty0
  (4):\penalty0 971--1013, 2018.

\bibitem[Buchholz(2022)]{buchholz2022kernel}
Simon Buchholz.
\newblock Kernel interpolation in {Sobolev} spaces is not consistent in low
  dimensions.
\newblock In \emph{Conference on Learning Theory}, pages 3410--3440. 2022.

\bibitem[Cao et~al.(2022)Cao, Chen, Belkin, and Gu]{cao2022benign}
Yuan Cao, Zixiang Chen, Misha Belkin, and Quanquan Gu.
\newblock Benign overfitting in two-layer convolutional neural networks.
\newblock In \emph{Advances in Neural Information Processing Systems}, pages
  25237--25250. 2022.

\bibitem[Caponnetto and De~Vito(2007)]{caponnetto2007optimal}
Andrea Caponnetto and Ernesto De~Vito.
\newblock Optimal rates for the regularized least-squares algorithm.
\newblock \emph{Foundations of Computational Mathematics}, 7:\penalty0
  331--368, 2007.

\bibitem[Dai and Xu(2013)]{dai2013approximation}
Feng Dai and Yuan Xu.
\newblock \emph{Approximation Theory and Harmonic Analysis on Spheres and
  Balls}, volume~23.
\newblock Springer, 2013.

\bibitem[Fischer and Steinwart(2020)]{fischer2020sobolev}
Simon Fischer and Ingo Steinwart.
\newblock Sobolev norm learning rates for regularized least-squares algorithms.
\newblock \emph{Journal of Machine Learning Research}, 21\penalty0
  (205):\penalty0 1--38, 2020.

\bibitem[Haas et~al.(2023)Haas, Holzm\"{u}ller, Luxburg, and
  Steinwart]{haas2023mind}
Moritz Haas, David Holzm\"{u}ller, Ulrike Luxburg, and Ingo Steinwart.
\newblock Mind the spikes: Benign overfitting of kernels and neural networks in
  fixed dimension.
\newblock In \emph{Advances in Neural Information Processing Systems}, pages
  20763--20826. 2023.

\bibitem[Hastie et~al.(2022)Hastie, Montanari, Rosset, and
  Tibshirani]{hastie2022surprises}
Trevor Hastie, Andrea Montanari, Saharon Rosset, and Ryan~J. Tibshirani.
\newblock Surprises in high-dimensional ridgeless least squares interpolation.
\newblock \emph{The Annals of Statistics}, 50\penalty0 (2), 2022.

\bibitem[Jacot et~al.(2018)Jacot, Gabriel, and Hongler]{jacot2018neural}
Arthur Jacot, Franck Gabriel, and Clement Hongler.
\newblock Neural tangent kernel: Convergence and generalization in neural
  networks.
\newblock In \emph{Advances in Neural Information Processing Systems}, pages
  8580--8589. 2018.

\bibitem[Li et~al.(2024)Li, Zhang, and Lin]{li2024kernel}
Yicheng Li, Haobo Zhang, and Qian Lin.
\newblock Kernel interpolation generalizes poorly.
\newblock \emph{Biometrika}, 111\penalty0 (2):\penalty0 715--722, 2024.

\bibitem[Liang and Rakhlin(2020)]{liang2020just}
Tengyuan Liang and Alexander Rakhlin.
\newblock Just interpolate: Kernel
  {\textquotedblleft}{Ridgeless}{\textquotedblright} regression can generalize.
\newblock \emph{The Annals of Statistics}, 48\penalty0 (3), 2020.

\bibitem[Lin et~al.(2020)Lin, Rudi, Rosasco, and Cevher]{lin2020optimal}
Junhong Lin, Alessandro Rudi, Lorenzo Rosasco, and Volkan Cevher.
\newblock Optimal rates for spectral algorithms with least-squares regression
  over {Hilbert} spaces.
\newblock \emph{Applied and Computational Harmonic Analysis}, 48\penalty0
  (3):\penalty0 868--890, 2020.

\bibitem[Lin et~al.(2022)Lin, Zhang, and Zhang]{lin2022reproducing}
Rongrong Lin, Haizhang Zhang, and Jun Zhang.
\newblock On reproducing kernel {Banach} spaces: Generic definitions and
  unified framework of constructions.
\newblock \emph{Acta Mathematica Sinica, English Series}, 38\penalty0
  (8):\penalty0 1459--1483, 2022.

\bibitem[Lin et~al.(2017)Lin, Guo, and Zhou]{lin2017distributed}
Shao-Bo Lin, Xin Guo, and Ding-Xuan Zhou.
\newblock Distributed learning with regularized least squares.
\newblock \emph{Journal of Machine Learning Research}, 18\penalty0
  (92):\penalty0 1--31, 2017.

\bibitem[Lin et~al.(2024)Lin, Chang, and Sun]{lin2024kernel}
Shao-Bo Lin, Xiangyu Chang, and Xingping Sun.
\newblock Kernel interpolation of high dimensional scattered data.
\newblock \emph{SIAM Journal on Numerical Analysis}, 62\penalty0 (3):\penalty0
  1098--1118, 2024.

\bibitem[Mei et~al.(2022)Mei, Misiakiewicz, and
  Montanari]{mei2022generalization}
Song Mei, Theodor Misiakiewicz, and Andrea Montanari.
\newblock Generalization error of random feature and kernel methods:
  Hypercontractivity and kernel matrix concentration.
\newblock \emph{Applied and Computational Harmonic Analysis}, 59:\penalty0
  3--84, 2022.

\bibitem[Nagarajan and Kolter(2019)]{nagarajan2019uniform}
Vaishnavh Nagarajan and J.~Zico Kolter.
\newblock Uniform convergence may be unable to explain generalization in deep
  learning.
\newblock In \emph{Advances in Neural Information Processing Systems}, pages
  11611--11622. 2019.

\bibitem[Parhi and Nowak(2023)]{parhi2023minimax}
Rahul Parhi and Robert~D. Nowak.
\newblock Near-minimax optimal estimation with shallow {ReLU} neural networks.
\newblock \emph{IEEE Transactions on Information Theory}, 69\penalty0
  (2):\penalty0 1125--1140, 2023.

\bibitem[Rakhlin and Zhai(2019)]{rakhlin2019consistency}
Alexander Rakhlin and Xiyu Zhai.
\newblock Consistency of interpolation with {Laplace} kernels is a
  high-dimensional phenomenon.
\newblock In \emph{Conference on Learning Theory}, pages 2595--2623. 2019.

\bibitem[Steinwart and Christmann(2008)]{steinwart2008support}
Ingo Steinwart and Andreas Christmann.
\newblock \emph{Support Vector Machines}.
\newblock Springer Science \& Business Media, 2008.

\bibitem[Steinwart and Scovel(2012)]{steinwart2012mercer}
Ingo Steinwart and Clint Scovel.
\newblock Mercer’s theorem on general domains: On the interaction between
  measures, kernels, and {RKHSs}.
\newblock \emph{Constructive Approximation}, 35:\penalty0 363--417, 2012.

\bibitem[Steinwart et~al.(2009)Steinwart, Hush, and
  Scovel]{steinwart2009optimal}
Ingo Steinwart, Don~R. Hush, and Clint Scovel.
\newblock Optimal rates for regularized least squares regression.
\newblock In \emph{Conference on Learning Theory}, pages 79--93, 2009.

\bibitem[Wainwright(2019)]{wainwright2019high}
Martin~J. Wainwright.
\newblock \emph{High-dimensional Statistics: A Non-asymptotic Viewpoint},
  volume~48.
\newblock Cambridge University Press, 2019.

\bibitem[Wendland(2004)]{wendland2004scattered}
Holger Wendland.
\newblock \emph{Scattered Data Approximation}, volume~17.
\newblock Cambridge University Press, 2004.

\bibitem[Yang and Zhou(2024)]{yang2024nonparametric}
Yunfei Yang and Ding-Xuan Zhou.
\newblock Nonparametric regression using over-parameterized shallow {ReLU}
  neural networks.
\newblock \emph{Journal of Machine Learning Research}, 25\penalty0
  (165):\penalty0 1--35, 2024.

\bibitem[Zhang et~al.(2017)Zhang, Bengio, Hardt, Recht, and
  Vinyals]{zhang2017understanding}
Chiyuan Zhang, Samy Bengio, Moritz Hardt, Benjamin Recht, and Oriol Vinyals.
\newblock Understanding deep learning requires rethinking generalization.
\newblock In \emph{International Conference on Learning Representations}, 2017.

\bibitem[Zhang et~al.(2009)Zhang, Xu, and Zhang]{zhang2009reproducing}
Haizhang Zhang, Yuesheng Xu, and Jun Zhang.
\newblock Reproducing kernel {Banach} spaces for machine learning.
\newblock \emph{Journal of Machine Learning Research}, 10\penalty0
  (95):\penalty0 2741--2775, 2009.

\bibitem[Zhang et~al.(2024)Zhang, Li, and Lin]{zhang2024optimality}
Haobo Zhang, Yicheng Li, and Qian Lin.
\newblock On the optimality of misspecified spectral algorithms.
\newblock \emph{Journal of Machine Learning Research}, 25\penalty0
  (188):\penalty0 1--50, 2024.

\bibitem[Zhou(2002)]{zhou2002covering}
Ding-Xuan Zhou.
\newblock The covering number in learning theory.
\newblock \emph{Journal of Complexity}, 18\penalty0 (3):\penalty0 739--767,
  2002.

\end{thebibliography}
\end{document}